\documentclass[11pt]{article}  

\usepackage{subfigure}
\usepackage{enumerate}
\usepackage{bm}
\usepackage{amsmath}
\usepackage{amsthm}
\usepackage{amssymb}
\usepackage{wrapfig}
\usepackage{array}
\usepackage{color}
\usepackage{algorithmic}
\usepackage{graphics,url}
\usepackage{subfigure} 
\usepackage{graphicx}

\definecolor{orange}{rgb}{0.0,1,0}
\definecolor{red}{rgb}{1,0,0}

\usepackage{subfigure,enumerate,bm,amsmath,amsthm,wrapfig,array,color}
\usepackage{graphics}

\newtheorem{theorem}{\bf Theorem}[]
\newtheorem{lemma}[theorem]{Lemma}

\newtheorem{corollary}[theorem]{Corollary}

\newcommand{\transp}{\ensuremath{^\text{\textsc{t}}}}

\newcommand{\mat}[1]{{\ensuremath{\textsc{#1}}}}

\def\matA{\mat{A}}
\def\matB{\mat{B}}

\def\matH{\mat{H}}
\def\matI{\mat{I}}

\def\matQ{\mat{Q}}

\def\matS{\mat{S}}

\DeclareMathSymbol{\Prob}{\mathbin}{AMSb}{"50}
\newcommand\remove[1]{}

\def\math#1{$#1$}

\def\mand#1{$$#1$$}
\def\frac#1#2{{#1\over #2}}

\def\mld#1{\begin{equation}
#1
\end{equation}}

\def\eqan#1{\begin{eqnarray*}
#1
\end{eqnarray*}}

\DeclareMathSymbol{\R}{\mathbin}{AMSb}{"52}
\def\cl#1{{\cal #1}}

\def\aa{{\mathbf a}}
\def\bb{{\mathbf b}}

\def\vv{{\mathbf v}}

\def\zz{{\mathbf z}}

\def\aa{{\mathbf a}}
\def\bb{{\mathbf b}}

\def\norm#1{{\|#1\|}}

\def\dotfil{\leaders\hbox to 1.5mm{.}\hfill}
\newcounter{rmnum}
\def\RN#1{\setcounter{rmnum}{#1}\uppercase\expandafter{\romannumeral\value{rmnum}}}
\def\rn#1{\setcounter{rmnum}{#1}\expandafter{\romannumeral\value{rmnum}}}

\providecommand\remove[1]{}
\DeclareMathSymbol{\Prob}{\mathbin}{AMSb}{"50}
\DeclareMathSymbol{\Exp}{\mathbin}{AMSb}{"45}

\usepackage[normalem]{ulem} % either use this (simple) or
\usepackage{soul}

\title{NP-Hardness and Inapproximability of Sparse PCA
}
\author{Malik Magdon-Ismail\\RPI CS Department, Troy, NY 12211.
{\sf magdon@cs.rpi.edu}
}

\begin{document}

\maketitle

\begin{abstract}%
\noindent
We give a reduction from {\sc clique} to establish that sparse PCA
is NP-hard. The reduction has a gap which we use to exclude
an FPTAS for sparse PCA (unless P=NP). Under weaker
complexity assumptions, we also exclude polynomial
 constant-factor approximation algorithms.
\end{abstract}

\def\spca{{\sc spca}}
\def\clique{{\sc clique}}
\def\dks{{\sc d\math{\kappa}s}}
\def\opt{{\sf OPT}}

\section{Introduction}

The earliest reference to principal components analysis (PCA) is in
\cite{P1901}. Since then, PCA has evolved into a classic tool
for data analysis.
A challenge for the interpretation of the principal 
components (or factors)
 is that they can be linear combinations of \emph{all} the original 
variables.
When the original variables have direct physical 
significance (e.g. genes in biological applications or assets in
financial applications) it is desirable to have factors 
which have loadings on only a small
number of the original variables. 
These interpretable factors are 
\emph{sparse principal components (\spca)}. 
There are many
heuristics for obtaining 
sparse factors~\cite{CJ95,TJU03,ZHT06,AEJL07,ABE08,MWA06a,SH08} 
as well as some approximation algorithms with provable guarantees
%\cite{APD2014,ourpaper}.
\cite{APD2014}.
Our goal in this short paper is to establish the NP-hardness and 
inapproximability of \spca{} using a  reduction from 
\clique{}.

The traditional formulation of sparse PCA is as cardinality
constrained variance maximization:
\begin{center}
\fbox{%
\begin{minipage}{0.72\textwidth}
\begin{description}\itemsep-3pt
\item[Problem:] \spca{} (sparse PCA)
\item[Input:] Symmetric matrix \math{\matS\in\R^{n\times n}}; 
sparsity  \math{r\ge 0}; variance \math{M\ge 0}.
\item[Question:] Does there exist  a unit vector \math{\vv\in\R^n}
with at most 
\math{r} non-zero elements (\math{\vv\transp\vv=1} and \math{\norm{\vv}_0\le r})
for which \math{\vv\transp\matS\vv\ge M}?
\end{description}
\end{minipage}
}
\end{center}
In the machine learning context, \math{\matS} is the covariance matrix for the
data and, when there is no sparsity constraint, the solution 
\math{\vv^*} is the top right singular vector of \math{\matS}.
A generalization of \spca{} is
the generalized eigenvalue problem:
maximize \math{\vv\transp\matS\vv} subject to
\math{\vv\transp\matQ\vv=1} and \math{\norm{\vv}_0\le r}.
This generalized eigenvalue problem is 
NP-hard \cite{MGWA08} (via a reduction from 
sparse regression which is known to be
NP-hard \cite{N95,FKT14}).
%The related problem of maximizing volume is also NP-hard~\cite{CM09}.
It is deeply embeded folklore that \spca{} is NP-hard.
%yet we have not found a explicit reduction to \spca. 
The 
importance of sparse factors in dimensionality reduction 
has been recognized in some early work
(the \emph{varimax} criterion \cite{K58} has been used to 
rotate the factors to encourage sparsity, and this
has been used in multi-dimensional scaling approaches to dimensionality
reduction \cite{S69,K64}).

{\bf Notation.}
\math{\matA,\matB,\ldots} are
matrices; \math{\aa,\bb,\ldots} are 
vectors; and, \math{G,H,\ldots} are graphs. The top eigenvalue of a matrix
\math{\matA} is \math{\lambda_1(\matA)}; \math{\norm{\matA}_2}
is the spectral norm. For an undirected
 graph \math{G}, its adjacency matrix
\math{\matA} is a (0,1)-matrix with \math{\matA_{ij}=1} whenever
edge \math{(i,j)} is in \math{G}. The spectral radius of a graph is the
spectral norm of its adjacency matrix 
(also the top eigenvalue \math{\lambda_1}). \math{\bm0} (resp.~\math{\bm1})
are vectors or matrices of only zeros (resp. ones); for example, 
\math{\bm1_{2\times 2}} is a \math{2\times 2} matrix of ones.

\section{Reduction from \clique{}}
\begin{center}
\fbox{%
\begin{minipage}{0.72\textwidth}
\begin{description}\itemsep-3pt
\item[Problem:] \clique{}
\item[Input:] Undirected graph \math{G=(V,E)}; clique size \math{K}.
\item[Question:] Does there exist a \math{K}-clique in \math{G}?
\end{description}
\end{minipage}
}
\end{center}
The reduction is fairly straightforward. Given the inputs
\math{(G,K)} for \clique{}, we construct the inputs
\math{(\matS,r,M)} for \spca{} as follows.
Let \math{\matS} be the adjacency matrix of \math{G}; let \math{r=K};
and, let \math{M=K-1}. Clearly the reduction is polynomial.
It only remains to prove that there is a \math{K}-clique in
\math{G} \emph{if and only if} there is a
\math{K}-sparse unit vector \math{\vv} for which 
\math{\vv\transp\matS\vv\ge K-1}.
We need the following lemma on the spectral radius (top eigenvalue)
of an adjacency matrix.
\begin{lemma}[\cite{CS57}]\label{lem:lam1}
Let \math{\mat{A}} be the adjacency matrix of a graph \math{H}
of order \math{\ell}.
If \math{H} is an \math{\ell}-clique, then
\math{\norm{\mat{A}}_2=\lambda_1(\mat{A})=\ell-1}; if 
\math{H} is not an \math{\ell}-clique,
then \math{\norm{\mat{A}}_2=\lambda_1(\mat{A})<\ell-1}.
\end{lemma}
We now prove the claim. Suppose \math{Q} is a \math{K}-clique in \math{G}
and let \math{\matS_Q} be the \math{K\times K} principal submatrix of
\math{\matS} corresponding to the nodes in \math{Q}. 
Let \math{\zz} be a unit-norm top eigenvector of \math{\matS_Q}, and let
\math{\vv(\zz)} be the vector with \math{K} non-zeros induced by 
\math{\zz}: the non-zeros in \math{\vv} are at the indices corresponding
to the nodes in \math{Q} and the values are the corresponding values
in \math{\zz}. Then,
\mand{
\vv\transp\matS\vv=\zz\transp\matS_Q\zz=\lambda_1(\matS_Q)=K-1,
}
where the last equality follows from Lemma~\ref{lem:lam1} because 
\math{\matS_Q} is the adjacency matrix of a \math{K}-clique. So, 
\math{\vv(\zz)} is a \math{K}-sparse unit vector for which 
\math{\vv\transp\matS\vv\ge K-1}.
Now, suppose that there is a unit-norm \math{K}-sparse \math{\vv}
for which \math{\vv\transp\matS\vv\ge K-1}. Let \math{\matS_Q} be the 
\math{K\times K}
principal submatrix of \math{\matS} corresponding to the non-zero entries
of \math{\vv} and let \math{\zz(\vv)} be the \math{K}-dimensional vector 
consisting only of the non-zeros of \math{\vv}. Let \math{Q} be the 
subgraph induced by the nodes corresponding to the non-zero indices of 
\math{\vv} (\math{\matS_Q} is the adjacency matrix of \math{Q}).
Then, \math{\vv\transp\matS\vv=\zz\transp\matS_Q\zz\ge K-1}, and so
\math{\lambda_1(\matS_Q)\ge K-1}.
By Lemma~\ref{lem:lam1} if \math{Q} is not a \math{K}-clique then
\math{\lambda_1(\matS_Q)< K-1}, so it follows that \math{Q} is a 
\math{K}-clique.
Clearly \spca{} is in NP and so it is NP-complete.

\section{Inapproximability of \spca}

We now provide evidence that there is 
no efficient approximation algorithm
for \spca.
First we rule out the possibility of a fully
 polynomial time approximation
scheme (FPTAS).
Given any instance \math{(\matS,r)} of \spca,
define 
\math{\opt(\matS,r)=\max_{\vv}\vv\transp\matS\vv} 
over unit-norm \math{r}-sparse \math{\vv}.
A \math{(1-\epsilon)}-approximation algorithm for \spca{}
produces a unit-norm \math{r}-sparse solution \math{\tilde\vv} for any given
instance \math{(\matS,r)}
satisfying \math{\tilde\vv\transp\matS\tilde\vv\ge(1-\epsilon)\opt(\matS,r)}.
An FPTAS is algorithm to compute a
\math{(1-\epsilon)}-approximation for 
\math{\epsilon>0} and every instance of \spca{} that is polynomial in
\math{n,r,\epsilon^{-1}}.
The next theorem establishes that there is no polynomial
\math{(1-O(1/r^2))}-approximation algorithm and hence no FPTAS.
\begin{theorem}[No FPTAS]\label{theorem:noFPTAS}
Unless P=NP, 
there is no polynomial time
\math{(1-\epsilon)}-approximation algorithm for \spca{} with
\mand{\epsilon<\epsilon^*(r)
=\frac{r+1}{2(r-1)}\left(1-\sqrt{1-\frac{8}{(r+1)^2}}\ \right)
=\frac{2}{r^2-1}+O(1/r^4).}
\end{theorem}
\begin{proof}
The proof essentially amounts to 
strengthening Lemma~\ref{lem:lam1} for the case that \math{H} is not
an \math{\ell}-clique. Specifically in  Lemma~\ref{lem:lam1}, if 
adjacency matrix \math{\matA\in\R^{\ell\times \ell}} is not the adjacency 
matrix of 
an \math{\ell}-clique, then we will show that 
\mld{
\lambda_1(\matA)\le
\frac{\ell-3}{2}+\frac{\ell+1}{2}\left(1-\frac{8}{(\ell+1)^2}\right)^{1/2}
=(\ell-1)(1-\epsilon^*(\ell)).
\label{eq:proof1}\tag{\math{*}}
}
Suppose that (\ref{eq:proof1}) holds whenever \math{\matH} is not an 
\math{\ell}-clique. For any \spca{} instance \math{(\matS,r)},
suppose the  polynomial algorithm \math{\cl{A}}
gives
a \math{(1-\epsilon)}-approximation with 
\math{\epsilon<\epsilon^*(r)}.
We show how to use \math{\cl{A}} to polynomialy decide
\clique. Given \math{(G,K)}, the inputs to \clique,
use our reduction to
construct \math{(\matS,K,K-1)}, the inputs to \spca. Now run 
algorithm \math{\cl{A}} on \math{(\matS,K)} to obtain \math{\tilde \vv} and
compute \math{x=\tilde \vv\matS\tilde\vv}.
If \math{x\ge (K-1)(1-\epsilon^*(K))} then \math{\opt(\matS,K)=K-1} and so
there is a \math{K}-clique in \math{G}; if 
\math{x< (K-1)(1-\epsilon^*(K))} then \math{\opt(\matS,K)<K-1}
(since we have a better than \math{(1-\epsilon^*(K))}-approximation)
and so there is no \math{K}-clique in \math{G}.

To prove (\ref{eq:proof1}), we first consider the adjacency matrix of
a complete graph minus one edge,
\mand{
\matA=
\left[
\begin{matrix}
\bm0_{2\times 2}&\bm1_{2\times(\ell-2)}\\
\bm1_{(\ell-2)\times 2}&\bm1_{\ell-2}\bm1_{\ell-2}\transp
-\matI_{(\ell-2)\times(\ell-2)}
\end{matrix}
\right]
}
By symmetry, the top eigenvector can be written
\math{\left[\begin{smallmatrix}x\bm1_{2}\\ y\bm1_{\ell-2}
\end{smallmatrix}\right]}.
The eigenvalue equation is
\mand{
\left[
\begin{matrix}
\bm0_{2\times 2}&\bm1_{2\times(\ell-2)}\\
\bm1_{(\ell-2)\times 2}&\bm1_{\ell-2}\bm1_{\ell-2}\transp
-\matI_{(\ell-2)\times(\ell-2)}
\end{matrix}
\right]
\left[\begin{matrix}x\bm1_{2}\\ y\bm1_{\ell-2}
\end{matrix}\right]
=
\lambda\left[\begin{matrix}x\bm1_{2}\\ y\bm1_{\ell-2}
\end{matrix}\right],
}
and we obtain the equations:
\eqan{
(\ell-2)y&=&\lambda x;\\
2x+(\ell-3)y&=&\lambda y.
}
Solving for \math{\lambda} gives the quadratic
\math{
\lambda^2-(\ell-3)\lambda-2(\ell-2)=0,}
and the positive root is
\mand{
\lambda=\frac{\ell-3}{2}+\frac{1}{2}\sqrt{(\ell+1)^2-8},
}
which is the expression in (\ref{eq:proof1}). Since the spectral radius is
strictly decreasing with edge-removal, we have proved the 
upper bound in (\ref{eq:proof1}).
\end{proof}

Under stronger (average-case) complexity assumptions we can also exclude
polynomial 
constant factor approximations for \spca. 
A natural optimization version of 
\clique{} is the densest-\math{K}-subgraph (\dks):
Given \math{(G,K)} find a subgraph \math{Q} on \math{K} nodes
with the maximum number of edges. There is evidence that \dks{} does not
admit efficient approximation algorithms~\cite{AAMM}.

Let \math{G} and \math{G'} be two graphs
on \math{n} vertices. Suppose that one of the graphs has
an \math{\ell}-clique 
and for the other graph, every subgraph on \math{\ell} vertices has at most
\math{\delta \ell(\ell-1)/2} edges for \math{0<\delta<1}. 
If one has a polynomial \math{\delta}-approximation 
algorithm for \dks{} then one can determine which of 
\math{G,G'} has the \math{\ell}-clique in polynomial time.
We show that if one has an \math{\alpha}-approximation algorithm for
\spca, then one can determine which of 
\math{G,G'} has the \math{\ell}-clique in polynomial time for 
\math{\delta\le\alpha^2}. This means that if there are no polynomial algorithms
to distinguish between graphs with \math{\ell}-cliques and
graphs whose \math{\ell} subsets are all below a density \math{\alpha^2}, 
then
there are no polynomial \math{\alpha}-approximation algorithms for 
\spca.

Suppose there is an \math{\alpha}-approximation 
algorithm for 
\spca. So, given any instance \math{(\matS,r)} of
\spca, in polynomial time one can 
construct a solution \math{\tilde\vv} for which 
\math{\tilde\vv\transp\matS\tilde\vv\ge\alpha\opt(\matS,r)}.
Let \math{G,G'} be the two graphs described above 
with \math{\delta=\alpha^2}. Note that 
\mand{\delta=\alpha^2<\alpha^2\frac{(\ell-1)}\ell+\frac1\ell,}
where the inequality is because \math{0<\alpha<1}.
Now, let \math{\matA} be the adjacency matrix of \math{G} and 
run the \math{\alpha}-approximation algorithm for \spca{} with inputs
\math{(\matA,\ell)} to produce a solution \math{\tilde\vv}. 
If \math{\tilde\vv\transp\matA\tilde\vv\ge\alpha(\ell-1)}, declare that
\math{G} contains the \math{\ell}-clique; otherwise declare that 
\math{G'} contains the \math{\ell}-clique. We prove that our algorithm
correctly identifies the graph with the  \math{\ell}-clique.

If \math{G} does contain the  \math{\ell}-clique, then 
\math{\opt(\matA,\ell)=\ell-1} and the output \math{\tilde\vv} will satisfy
\math{\tilde\vv\transp\matA\tilde\vv\ge\alpha(\ell-1)}
(because it is an \math{\alpha}-approximation) and so we will
correctly identify \math{G} to have  \math{\ell}-clique.
Now suppose that \math{G} does not contain the  \math{\ell}-clique.
So, every \math{\ell}-node subgraph in \math{G} has at most 
\math{e\le\delta\ell(\ell-1)/2} edges. 
We now use the bound on the spectral radius of a graph with 
\math{e} edges from~\cite{H88}: 
\math{\norm{\matA}_2\le\sqrt{2e-n+1}}, and since 
\math{e\le\delta\ell(\ell-1)/2}, we have that
\eqan{
\norm{\matA}_2
&\le& 
\sqrt{\delta\ell(\ell-1)-\ell+1}\\
&=&
\sqrt{\alpha^2\ell(\ell-1)-\ell+1}\\
&<&
\sqrt{\left(\alpha^2\frac{(\ell-1)}{\ell}+\frac1\ell\right)\ell(\ell-1)-\ell+1}
\\
&=&\alpha(\ell-1).
}
Since \math{\norm{\matA}_2<\alpha(\ell-1)}, we will correctly identify
\math{G'} to have the \math{\ell}-clique.
The conclusion is summarised in the
following theorem.
\begin{theorem}\label{theorem:dks2spca}
A polynomial \math{\alpha}-approximation algorithm for
\spca{} gives a polynomial algorithm to distinguish between
two graphs on \math{n} vertices, one of which contains an
\math{\ell}-clique and the other with every subset of 
\math{\ell} nodes having at most \math{\alpha^2\ell(\ell-1)/2} edges
(for any \math{(n,\ell)}).
\end{theorem}
Under a variety of complexity assumptions it is known that one 
cannot efficiently distinguish between graphs with \math{\ell}-cliques 
and graphs in which all subsets of size \math{\ell} are sparse (for
varying degrees of sparseness).
\begin{theorem}[No constant factor approximation for \dks~\cite{AAMM}]
\label{AAMM}
Let \math{1>\delta>0} be any
constant approximation factor.
Let \math{G} and \math{G'} be two graphs
on \math{\ell^2} vertices. One of the graphs has
an \math{\ell}-clique 
and for the other graph, every subgraph on \math{\ell} vertices has at most
\math{\delta \ell(\ell-1)/2} edges. 
Suppose there is no polynomial time algorithm for solving the hidden 
clique problem for a planted clique of size \math{n^{1/3}}.
Then, there is no polynomial algorithm to determine which of
\math{G,G'} has the \math{\ell}-clique.
\end{theorem}
Using Theorem~\ref{theorem:dks2spca} with Theorem~\ref{AAMM},
\begin{corollary}[No constant factor approximation for \spca]\label{theorem:noConstant}
Suppose there is no polynomial time algorithm for solving the hidden 
clique problem for a planted clique of size \math{n^{1/3}}.
Then, for any constant \math{0<\alpha<1}, there is no
polynomial time \math{\alpha}-approximation algorithm for 
\spca.
\end{corollary}

%\subsubsection*{Acknowledgments}
{\small
\bibliographystyle{abbrv}
\bibliography{mypapers,spca,masterbib} 
}

\end{document}